\documentclass[letterpaper, 10 pt, conference]{ieeeconf}  

\IEEEoverridecommandlockouts                              
\usepackage{graphicx, subfig}
\usepackage{amsmath}
\usepackage{amssymb}
\usepackage{multirow}
\usepackage{color}
\usepackage[usenames,dvipsnames]{xcolor}
\usepackage{algorithm}
\usepackage{algorithmic}

\usepackage{caption}
\newcommand{\T}{\mathcal{T}}

\newcommand{\B}{\mathcal{B}}
\newcommand{\Prod}{\mathcal{P}}

\newcommand{\G}{\mathbf{G}}
\newcommand{\F}{\mathbf{F}}
\newcommand{\U}{\mathbf{U}}
\newcommand{\X}{\mathbf{X}}
\newcommand{\Nat}{\mathbb{N}}
\newcommand{\Real}{\mathbb{R}}
\newcommand{\ie}{\emph{i.e.}, }
\newcommand{\eg}{\emph{e.g.}, }
\newcommand{\run}{\rho}
\newcommand{\runf}{\rho_{\mathrm{fin}}}
\newcommand{\runp}{\rho_{\mathrm{pfix}}}
\newcommand{\barrunp}{\bar{\rho}_{\mathrm{pfix}}}
\newcommand{\runprod}{\varrho}
\newcommand{\runprodp}{\varrho_{\mathrm{pfix}}}
\newcommand{\sur}{\pi_{\mathrm{sur}}}
\newcommand{\pot}{\mathsf{pot}}
\newcommand{\pref}{\mathsf{pref}}
\newcommand{\short}{I}
\newcommand{\vis}{v}
\newcommand{\Vis}{V}
\newcommand{\dist}{W^*}

\newcommand{\distp}{W^*_\Prod}
\newcommand{\distpi}{W^*_{\Prod\pi}}
\newcommand{\distvarphi}{W^*_{\Prod\phi}}

\newcommand{\attr}{\mathsf{attr}_\Prod}
\newcommand{\sinf}{S_{\Prod\pi}^\infty}
\newcommand{\finf}{F_{\Prod}^\infty}
\newcommand{\maxpot}{\mathsf{pot'}}

\newtheorem{definition}{Definition}
\newtheorem{problem}{Problem Formulation}
\newtheorem{theorem}{Theorem}
\newtheorem{example}{Example}
\newtheorem{assumption}{Assumption}
\newtheorem{lemma}{Lemma}
\newtheorem{corollary}{Corollary}

\title{\LARGE \bf Attraction-Based Receding Horizon Path Planning \\ with Temporal Logic Constraints}

\author{M\'aria Svore\v{n}ov\'a, Jana T\r{u}mov\'a, Ji\v{r}\'i Barnat and Ivana \v{C}ern\'a
\thanks{This work was partially supported by grants no. GD102/09/H042, GAP202/11/0312, and LH11065.}
\thanks{The authors are with Faculty of Informatics, Masaryk University, Brno, Czech Republic. Email:
        {\tt\footnotesize svorenova@mail.muni.cz, \{xtumova,barnat,cerna\}@fi.muni.cz.}}
}

\begin{document}

\maketitle
\thispagestyle{empty}
\pagestyle{empty}


\begin{abstract}
Our goal in this paper is to plan the motion of a robot in a partitioned environment with dynamically changing, locally sensed rewards. We assume that arbitrary assumptions on the reward dynamics can be given. The robot aims to accomplish a high-level temporal logic surveillance mission and to locally optimize the collection of the rewards in the visited regions. These two objectives often conflict and only a compromise between them can be reached. We address this issue by taking into consideration a user-defined \emph{preference function} that captures the trade-off between the importance of collecting high rewards and the importance of making progress towards a surveyed region. Our solution leverages ideas from the automata-based approach to model checking. We demonstrate the utilization and benefits of the suggested framework in an illustrative example.
\end{abstract}


\section{Introduction}

In this paper, we consider the problem of robot path planning (see, \eg \cite{lavalle2006planning} for an overview) with more complex missions than "Go from $A$ to $B$ while avoiding obstacles.". Recently, different versions of temporal logics, such as Linear Temporal Logic (LTL), Computation Tree Logic (CTL), or $\mu$-calculus have been successfully employed to specify such robotic missions~\cite{ Antoniotti95,Loizou04,marius-tac2008, sertac-cdc2009, hadas09TL, fainekos09TL, LaAnBe-ACC10}. We have chosen LTL~\cite{FirstLtlDef,principles} as the specification means for its resemblance to natural language and its ability to express interesting robot behavior, such as "Repeatedly survey regions $A$ and $B$ while avoiding dangerous regions. Make sure, that $A$ is always visited in between two successive visits to $B$ and vice versa.".

We assume that the robot motion in the environment is modeled as a transition system, which is obtained by partitioning the environment into regions (for instance using well-known triangulations and rectangular partitions). Each region is modeled as a state of the transition system and the robot's capability to move between the regions as transitions between the corresponding states. Our transition system is deterministic, \ie a control input for the robot is the next region (state) to be visited. Moreover, the transition system is weighted, \ie each transition is equipped with the time duration this transition takes.

The robot's task is to collect rewards that dynamically appear, disappear and change their values in the environment regions and that can be sensed only within a certain vicinity of the robot's current state. A traditional approach to this kind of problem, \ie an optimization problem defined on a dynamically changing plant, is model predictive control (MPC)~\cite{mpc}. The method is based on iterative re-planning and optimization of a cost function over a finite horizon and hence, it is also called receding horizon control.

In this work, we focus on interconnecting the receding horizon control with the synthesis of a path that is provably correct with respect to a given temporal logic formula. This idea appeared in~\cite{Nok2009,nok-hscc2010}, where the receding horizon approach was employed to fight the high computational complexity of reactive motion planning with a specification in GR(1) fragment of LTL. However, the authors did not consider any rewards collection to be optimized. In contrast, the authors in~\cite{Dennis2010} addressed a similar problem that we do. They assumed a deterministic weighted transition system with locally sensed rewards changing according to an unknown dynamics. While they required the satisfaction of an LTL mission, they also targeted to collect maximal rewards locally, within a given horizon. These two goals often cannot be reached simultaneously. If the robot primarily collects high rewards, the mission might never be satisfied and vice versa, if the robot is planned to accomplish the mission, the collected rewards might become low. The authors utilized ideas from the automata-based approach to model checking in order to iteratively find a local path maximizing the collected rewards among the local paths that ensure that a step towards the mission satisfaction is made. This way, they managed to compromise between the two goals.

Our work can be seen as a different, generalized approach to the above problem. We allow the trade-off between the two goals to be partially driven by user-defined preferences that may dynamically change during the execution of the robot. In particular, we assume an LTL mission that includes surveillance of a set of regions and a user-defined \emph{preference function} expressing the desired trade-off between the surveillance and the reward collection given the history of the robot motion. In other words the preference function determines in each moment whether moving towards a surveyed place or optimization of the collected rewards is of a higher priority. Whereas the local path planned in~\cite{Dennis2010} always guarantees progress towards the satisfaction of the mission, in our case this progress may be deliberately postponed (for a finite amount of time) if the collection of the rewards is prioritized. For example, consider a garbage truck that is required to periodically visit two garbage disposal plants $A$ and $B$ and to arrive to a plant as fully loaded as possible. In~\cite{Dennis2010}, each local plan for the truck would send the truck closer to $A$ (or $B$, respectively) and the truck might arrive half-empty. In contrast, through the preference function, we can define that collecting the garbage is the primary target until the truck is full enough to drive to a plant and that once it is, driving towards $A$ (or $B$, respectively) becomes the priority. Besides that, we generalize the problem from~\cite{Dennis2010} in the following sense. The authors there assumed that the reward dynamics is completely unknown. Therefore, when planning, they estimate that the rewards collected along a planned local path would be equal to the sum of the rewards that are currently seen at the states of this path and they aim to maximize it. We consider that arbitrary assumptions on the reward dynamics might be given and we estimate the rewards collected along a planned local path accordingly. We also allow for a broader class of optimization functions. 

In our solution, we leverage ideas from the automata-based approach to model-checking to provably guarantee the satisfaction of the mission and we introduce several extensions that allow us to support both the preference function and the arbitrary assumptions on the reward dynamics. We build a so-called product automaton that captures all the runs of the transition system that satisfy the mission. We employ the preference function to compute the \emph{attraction} of states in the product automaton and at each time, we choose the most attractive state to be visited next. While the value of the preference function is low, the robot is primarily driven by the sensed rewards. However, as the preference function grows, the surveillance is prioritized and the attraction forces the robot to move not only towards the surveyed regions, but also towards accepting states of the product automaton, \ie towards the satisfaction of the global specification.

Our contribution can be summarized as follows. We develop a general framework for robot motion planning with high-level LTL mission specifications and locally optimal reward collection with respect to given reward dynamics assumptions and local rewards sensing. We introduce a novel approach that allows to prescribe whether the rewards collection or the mission progress are of a higher interest. We present several illustrative examples and simulation results to demonstrate the usability of  our approach.

The rest of the paper is organized as follows. In Section~\ref{sec:prelims} we review necessary preliminaries. In Section~\ref{sec:pf}, the problem is described in detail and stated formally. In Section~\ref{sec:solution}, we present the solution, correctness and completeness proofs and discussions on the solution optimality. In Section~\ref{sec:case} a case study is introduced and we conclude in Section~\ref{sec:conclusion}.


\section{Notation and Preliminaries}
\label{sec:prelims}

In this section we introduce notation and preliminaries used throughout the paper. 

Given a set $\mathsf{S}$, we denote by $\mathsf{S}^+$ and $\mathsf{S}^\omega$ all finite, nonempty and all infinite sequences of elements from $\mathsf{S}$, respectively.

\begin{definition}[Weighted Deterministic Transition System]
\label{def:dts}
A \emph{weighted deterministic transition system (TS)} is a tuple $\T = (Q, q_0, T , \Pi, L, W)$, where
\begin{itemize}
\item $Q$ is a finite set of states;
\item $q_0\in Q$ is an initial state;
\item $T \subseteq Q \times Q$ is a transition relation;
\item $\Pi$ is a set of atomic propositions;
\item $L: Q \rightarrow 2^\Pi$ is a labeling function; and
\item $W: T \to \Real_{>0}$ is a weight function.
\end{itemize}
\end{definition}
The states of the transition system represent the regions of the environment and the transitions represent the robot's capabilities to move between them. We assume that there is a transition from each state. The atomic propositions are properties that are either true or false in each region of the environment, for instance "This region is a pickup/delivery location.". The labeling function $L$ assigns to each state the set of atomic propositions that hold true in this state. The weight function assings to each transition the amount of time that this transition takes. If the robot is in a state $q$ at time $t$ and follows a transition $(q,q') \in T$, then it is in the state $q'$ at time $t+ W\big((q,q')\big)$. The time spent in states is $0$.

A \emph{run} of $\T$ is an infinite sequence $\run=q_0q_1\ldots$ such that $q_0$ is the initial state and $(q_i,q_{i+1}) \in T$, for all $i\geq 0$. A~\emph{finite run} $\runf = q_i \ldots q_j $ of $\T$ is a finite subsequence of a run $\run = q_0\ldots q_i \ldots q_j\ldots $ of $\T$. A \emph{run prefix} $\runp$ of $\T$ is a finite run that originates at the initial state $q_0$. For simplicity, we denote by $q \in \run$ ($q \in \runf)$ the fact that the state $q$ occurs in the run $\run$ (the finite run $\runf$). Associated with a run $\run = q_0q_1\ldots$ (and a run prefix $\runp=q_0\ldots q_n$) there is a sequence of time instances $t_0t_1\ldots$ (and $t_0\ldots t_n$), where $t_0 = 0$, and $t_i$ denotes the time at which the state $q_i$ is reached $(t_{i+1} = t_i + W( q_i, q_{i+1}))$. A run $\rho=q_0q_1\ldots$ generates a unique \emph{word} $\omega(\rho) = L(q_0)L(q_1)\ldots$. A~\emph{control strategy} $C: Q^+ \rightarrow Q$ for $\T$ assigns the next state to be visited to each run prefix of~$\T$. The run generated by $C$  is $\rho = q_0q_1\ldots$, such that $q_i = C(q_0\ldots q_{i-1})$, for all $i \geq 1$.

With a slight abuse of notation, we use $W(\runf)$, where $\runf = q_i\ldots q_j$ is a finite run of $T$, to denote the total weight of the sequence of transitions $(q_i,q_{i+1}), \ldots, (q_{j-1},q_j)$, \ie  $W(\runf) = \sum_{k=i}^{j-1} W\big((q_k,q_{k+1})\big)$. Furthermore, we define $\dist(q_i, q_j)$ as the minimum weight of a finite run from $q_i$ to $q_j$. In particular, $\dist(q_i,q_i)=0$, and $\dist(q_i,q_j)=\infty$ if there does not exist a finite run from $q_i$ to $q_j$. 

\begin{definition}[Linear Temporal Logic]
\label{def:ltl}
A \emph{linear temporal logic (LTL) formula} $\phi$ over the set of atomic propositions $\Pi$ is defined according to the following rules:
$$\phi ::= \top \mid  \pi \mid \neg \phi \mid \phi \vee \phi \mid \phi \wedge \phi \mid \X \, \phi \mid \phi \, \U \, \phi \mid \G\, \phi \mid \F\, \phi,$$
where $\top$ is always true, $\pi \in \Pi$ is an atomic proposition, $\neg$~(negation), $\vee$ (disjunction) and $\wedge$ (conjunction) are standard Boolean connectives, and $\X$ (next), $\U$ (until), $\G$ (always) and $\F$~(eventually) are temporal operators.
\end{definition}
The semantics of LTL is defined over infinite sequences over~$2^\Pi$, such as those generated by the transition system $\T$ from Def.~\ref{def:dts}. Assume that $\phi$, $\phi_{1}$, and $\phi_{2}$ are LTL formulas over $\Pi$ and $\omega = \omega(0)\omega(1)\ldots \in {(2^\Pi)}^\omega$ is a word generated by a run $\rho$ of $\T$. The word $\omega$ satisfies an atomic proposition $\pi$ if $\pi$ holds in the first position of $\omega$, \emph{i.e.}, if $\pi \in \omega(0)$. The formula $\X \, \phi$ states that $\phi$ needs to hold next, \ie for the word $\omega(1)\ldots$. The formula $\phi_1\,\U\, \phi_2$ means that $\phi_2$ is true eventually, while $\phi_1$ is true at least until $\phi_2$ becomes true. Formulas $\G \,\phi$ and $\F \, \phi$ state that $\phi$ holds always and eventually, respectively. More expressiveness can be achieved by combining the operators. A detailed description of LTL can be found in \cite{principles}. As expected, a run $\run$ of $\T$ satisfies $\phi$ if and only if the word $\omega(\rho)$ generated by $\run$ satisfies $\phi$.

\begin{definition}[B\"uchi Automaton]
\label{def:buchi}
A \emph{B\"uchi automaton (BA)} is a tuple $\B=(S,s_{0},\Sigma,\delta,F)$, where
\begin{itemize}
\item $S$ is a finite set of states;
\item $s_{0} \in S$ is an initial state;
\item $\Sigma$ is an input alphabet;
\item $\delta \subseteq S \times \Sigma \times {S}$ is a transition relation; and
\item $F \subseteq S$ is a set of accepting states.
\end{itemize}
\end{definition}
The semantics of a B\"uchi automaton is defined over infinite \emph{input words}. Note that if $\Sigma = 2^\Pi$, then the input words are infinite sequences of sets of atomic propositions, such as those generated by $\T$. A \emph{run} of $\B$ \emph{over} an input word $\sigma = a_0a_1\ldots \in \Sigma^\omega$ is a sequence of states $\varrho = s_0s_1\ldots$ such that $s_0$ is the initial state and $\big(s_i,a_i, s_{i+1}\big) \in \delta$, for all $i\geq 0$. A run $\varrho$ is \emph{accepting} if and only if a state from $F$ appears in $\varrho$ infinitely many times. A word $\sigma$ is accepted by the B\"uchi automaton if there exists an accepting run over~$\sigma$.

For any LTL formula $\phi$ over $\Pi$, there exists a B\"uchi automaton $\B_\phi$ with input alphabet $2^\Pi$ accepting all and only the words satisfying formula $\phi$. Algorithms for translation of an LTL formula into a corresponding B\"uchi automaton were proposed \cite{FastLtlToBa}, and several tools are available \cite{ltl2dstar}.

\begin{definition}[Weighted Product Automaton]
\label{def:product}
A \emph{weighted product automaton} between a TS $\T=(Q,q_0,T,\Pi,L,W)$ and a BA $\B_\phi=(S,s_{0},2^\Pi, \delta,F)$ is a tuple $\Prod=\T \times B_\phi = (S_\Prod,s_{\Prod0},\delta_{\Prod},F_{\Prod},W_\Prod)$, where
\begin{itemize}
\item $S_\Prod=Q\times S$ is a set of states;
\item $s_{\Prod0}=(q_0,s_0)$ is the initial state;
\item $\delta_{\Prod} \subseteq S_\Prod \times {S_\Prod}$ is a transition relation, where $\big((q,s),(q',s')\big)\in \delta_\Prod$ if and only if $(q,q')\in T$ and $(s,L(q),s') \in \delta$;
\item $F_{\Prod}=Q\times F$ is the set of accepting states; and
\item $W_{\Prod}: \delta_{\Prod} \to \mathbb{R}_{>0}$ is a weight function, where $W_{\Prod}\big(\big((q,s),(q',s')\big)\big) = W\big((q,q')\big)$, for all $\big((q,s),(q',s)\big) \in \delta_\Prod$.
\end{itemize}
\end{definition}

Note that the product automaton defined above is a weighted version of a standard B\"uchi automaton with a trivial alphabet that is thus omitted. We denote by $\alpha(\varrho_\Prod)$ the projection of a run $\varrho_\Prod$ of $\Prod$ onto its first components, \ie $\alpha\big((q_0,s_0)(q_1,s_1)\ldots\big) = q_0q_1\ldots$.
An accepting run $\varrho_\Prod$ of the product automaton $\Prod$ can be projected onto a run $\alpha(\varrho_\Prod)$ of $\T$ that satisfies the formula $\phi$, and vice versa, if $\rho = q_0q_1\ldots$ is a run of $\T$ satisfying $\phi$, then there exists an accepting run $\varrho_\Prod =(q_0,s_0)(q_1,s_1)\ldots$ of $\Prod$.

The product automaton can be also viewed as a transition system $\T_\Prod = (S_\Prod,s_{\Prod0},\delta_{\Prod},\Pi, L_\Prod,W_\Prod)$, where $L_\Prod\big((q,s)\big) = L(q)$, for all $(q,s) \in S_\Prod$. Hence, the objects that are defined on a transition system are defined on the product automaton in the expected way. Namely, we use $W_\Prod\big(\varrho_{\Prod\mathrm{fin}} \big)$, $\distp\big(p_i,p_j\big)$ and 
$C_\Prod\big(\varrho_{\Prod\mathrm{pfix}}\big)$ to denote the total weight of a finite run $\varrho_{\Prod\mathrm{fin}}$, the minimum weight between states $p_i$, $p_j$ and the control strategy for $\Prod$, respectively.


\section{Problem Formulation}
\label{sec:pf}

Consider a robot moving in a partitioned environment modeled as a weighted deterministic transition system. The states of the transition system correspond to individual regions of the environment and the transition between them model the robot motion capabilities.
Assume, that there is a dynamically changing non-negative real-valued \emph{reward} associated with each state of the transition system. The robot senses the rewards in its close proximity and collects the rewards as it visits the regions of the environment, \ie as the states of the transition system change. Moreover, the robot is given a high-level LTL \emph{mission}. The problem addressed in previous literature~\cite{Dennis2010} is to design a control strategy that (1)~guarantees the satisfaction of the mission and (2)~locally maximizes the collected rewards.

We focus on a different version of the above problem allowing for partial regulation of the trade-off between the two objectives. In particular, first, we consider a user-defined \emph{preference function} that, given a history of robot's movement, expresses whether moving closer to a region under surveillance or collecting rewards is prioritized. Second, we consider arbitrary reward dynamics that might be unknown, known partially or even fully. We capture the concrete reward dynamics assumptions through a so-called \emph{state potential function}. 
The problem we address is to design a control strategy that (1)~guarantees the satisfaction of the mission, (2)~locally optimizes the collection of rewards, and (3)~takes into consideration the preference function and the reward dynamics assumptions.

We formalize the problem as follows.
The robot motion in the environment is given as a TS $\T~=~(Q,q_0,T,\Pi,L,W)$ (Def.~\ref{def:dts}). The rewards can be sensed at time $t_k$ within the \emph{visibility range} $\vis \in \Real_{>0}$ from the robot's current position~$q_k$. We denote by $\Vis(q_k) = \{q \mid \dist(q_k,q) \leq v\}$ the set of states that are within the visibility range~$\vis$ from $q_k$ (assuming that $q \in V(q_k)$, for all $(q_k,q)\in T$) and by $R: Q \times Q^+ \rightarrow \Real_{\geq 0}$ the \emph{reward function}, where $R(q,q_0\ldots q_k)$ is the reward sensed in the state $q$ at time~$t_k$ after executing the run prefix $q_0\ldots q_k$. Note that $R(q,q_0\ldots q_k)$ is defined iff $q \in \Vis(q_k)$ and it is known only at time $t_k$ (and later), not earlier.

A user-defined \emph{planning horizon} and a \emph{state potential function} are employed to capture user's assumptions about the reward dynamics and her interests. For instance, the values of the rewards may increase or decrease at most by~1 during 1 time unit, they may appear according to a probabilistic distribution, or their changes might be random. The user might have full, partial or no knowledge of the reward dynamics. The rewards might disappear once they are collected by the robot, or they might not. The user might be interested in the maximal, expected, or minimal sum of rewards that can be collected from a given state during a finite run whose weight is no more than the planning horizon. The concrete definitions of the planning horizon and the state potential function are meant to be specifically tailored for different cases.
Formally, the horizon is $h \in \Real_{>0}$, $h\geq \max_{(q,q')\in T} W(q,q')$ and the state potential function is $\pot: Q \times Q^+ \times \Real_{>0} \rightarrow \Real_{\geq0},$ where $\pot(q,q_0\ldots q_k,h)$ is the potential of the state $q$ at time~$t_k$. 
More precisely, the value of $\pot(q,q_0\ldots q_k,h)$ is defined for all~$q$, where $(q_k,q) \in T$ and captures the rewards that can be collected after execution of the run prefix $q_0\ldots q_k$ during a finite run $\runf \in P_\mathrm{fin}(q,q_k,h)$, where
\addtolength{\abovedisplayskip}{-0.1em}
\addtolength{\belowdisplayskip}{-0.1em}
\begin{align*}
P_\mathrm{fin} & (q,q_k,h) = \{ \runf \mid \  \runf \text{ is a finite run of $\T$, such that } \\
& \text{(i) } \runf \text { originates at } q; \\ & \text{(ii) }  W(\runf) + W\big((q_k,q)\big) \leq h; \text{ and}\\ & \text{(iii) } \text{the states that appear in } \runf \text{ belong to } \Vis(q_k)\}.
\end{align*}

Note, that the visibility range $v$ and the planning horizon $h$ are independent. Whilst $v$ determines the set of states whose rewards are visible from the current state $q_k$, $h$ gives the maximal total weight of a planned finite run $\runf$ within $\Vis(q_k)$, which can be even greater than $v$.

\begin{example}
The function stating that the potential of $q$ is the maximal sum of rewards that can be collected from $q$ assuming that the rewards do not change while the robot can sense them and that they disappear once collected is
$$\pot(q,q_0\ldots q_k,h) = \max\limits_{\runf \in P_\mathrm{fin} (q,q_k,h)} \sum_{q' \in  \runf} R(q',q_0\ldots q_k).$$
 In fact, this is how authors in~\cite{Dennis2010} estimate the amount of rewards collected on a local path. 
\end{example}

To define our problem, we assume that there is a set of regions labeled with a so-called \emph{surveillance proposition} $\sur \in \Pi$ and a part of the mission is to periodically fulfill the surveillance proposition by visiting one of those regions. The missions are then expressed as LTL formulas of form
\begin{equation}
\phi = \varphi \wedge \G\,\F\, \sur,
\label{eq:formula}
\end{equation}
where $\varphi$ is an arbitrary LTL formula over $\Pi$. The subformula $\G \, \F \, \sur$ states that the surveillance proposition $\sur$ has to be visited always eventually, \ie infinitely many times. 
Note, that formulas $\phi = \varphi \wedge \G\,\F\, \top$ hold true if and only if $\varphi$ hold true and therefore the prescribed form does not restrict the full LTL expressivity.

The user can partially guide whether the robot should collect high rewards or whether it should rather make a step towards the satisfaction of the surveillance proposition $\sur$ through a \emph{preference function}. For example, the preference function can grow linearly with time since the latest visit to~$\sur$, meaning that going towards $\sur$ gradually gains more importance. In contrast, the value of the preference function can stay low until the latest visit to $\sur$ happened no later than 100 time units ago and after that increase rapidly, expressing that the robot is preferred to collect rewards for 100 time units and then to move towards $\sur$ quickly.

Formally, the preference function $\pref: Q^+ \rightarrow \Real_{\geq0}$ assigns a non-negative real value to each executed run prefix $q_0\ldots q_k$ of $\T$ (possibly) taking into account the current values of the state potential function. 

\begin{example} An example of a preference function is
$$\pref(q_0\ldots q_k) = 0.01 \cdot W_{i} \cdot \max\limits_{(q_{k},q)\in T}\pot(q,q_0\ldots q_k,h),$$
 where $W_i = W(q_i \ldots q_k)$, such that $\sur \in L(q_i)$, and $\sur \not \in L(q_j)$, for all $i<j\leq k$. If the surveyed state is being avoided, the total weight $W_i$ since the last visit to a surveyed state gradually grows and eventually, the value of $\pref(q_0\ldots q_k)$ overgrows the value of $\pot(q,q_0\ldots q_k,h)$ for all $q$.
\end{example}

A \emph{shortening indicator function} $I$ indicates whether a transition leads the robot closer to a state subject to surveillance.
$\short: T \rightarrow \{0,1\}$ is defined as follows:
\begin{equation*}\label{eq:short}
 \short\big((q,q')\big) = \begin{cases}
        1  & \text{if} \min\limits_{q_\pi \in Q_\pi} \dist(q',q_\pi) < \min\limits_{q_\pi \in Q_\pi} \dist(q,q_\pi),\\
        0 & \text{otherwise},
        \end{cases}
\end{equation*}
where $(q,q')\in T$ and $Q_\pi = \{q_\pi \mid \sur \in L(q_\pi)\}$. 

We are now ready to formally state our problem.
\begin{problem}\label{problem}
\textbf{Given}
the robot motion model $\T = (Q,q_0,T,\Pi,L,W)$;
the surveillance proposition $\sur \in \Pi$;
the visibility range $\vis$;
the reward $R(q,q_0\ldots q_k)$ at time $t_k$, for all $q \in \Vis(q_k)$;
the planning horizon $h$;
the state potential function $\pot$;
the LTL formula $\phi$ over $\Pi$ (Eq.~\ref{eq:formula}); and
the preference function $\pref$,
\textbf{find} a control strategy $C$, such that
\begin{itemize}
\item[(i)] the run generated by $C$ satisfies the mission $\phi$ and
\item[(ii)] assuming that $q = C(q_0 \ldots q_k)$, the cost function
\begin{equation}
\pot(q,q_0\ldots q_k ,h) + I\big((q_k,q)\big)\cdot\pref(q_0\ldots q_k)
\label{eq:opt}
\end{equation}
is maximized at each time $t_k$.
\end{itemize}
\end{problem}

Intuitively, condition (ii) is interpreted as follows. At each time, the aim is to go to the state with the best trade-off between the amount of potentially collected rewards and the importance of fast surveillance. The higher the value of the preference function, the more likely a state closer to $\sur$ is to be chosen. Note that, in general, the satisfaction of the condition~(ii) may cause violation of the objective~(i). Our goal is thus to provably guarantee accomplishment of the mission and to maximize Eq.~\ref{eq:opt}, if possible.

Our approach leverages some ideas from the automata-based solution from~\cite{Dennis2010}. However, several issues have to be overcome to support the user-defined trade-off as it will become clear in the following section. The solution consists of two consecutive steps. The first one is an offline preparation before the deployment of the system. It involves a construction of a BA for the given LTL mission and its product with the TS.
The offline algorithm assigns two Boolean indicators to each transition of the product automaton, which indicate whether the transition induces a progress to a subgoal, \ie a surveyed state of the transition system and both a surveyed state and an accepting state of the product automaton, respectively.
In the second step, an online feedback algorithm, which determines the next state to be visited by the robot, is iteratively run. In each iteration, attractions of the states of the product automaton are computed. The repeated choices of the maximal attraction states lead to an eventual visit not only to a surveyed state, but also to an accepting state of the product automaton, assuming that the following holds:

\begin{assumption}
For each run $q_0q_1\ldots$, with the property that $\exists n_1$, $\forall m> n_1$: $\sur \not \in L(q_m)$, it holds that $\exists n_2$, $\forall m> n_2$:  $\pref(q_0\ldots q_m)>\pot(q,q_0\ldots q_m,h)$, for all $q$, where $(q_m,q)\in T$.
\label{assump}
\end{assumption}

As we will show in Sec.~\ref{sec:solution_cor} the satisfaction of the LTL mission is guaranteed provided that the above assumption is true. From now on, we assume that Assump.~\ref{assump} holds. Intuitively, it says that if a visit to a surveyed state is postponed for a long time, the value of the preference function overweights the value of the state potentials. Note that this is, in fact, quite natural. It only captures the fact, that the user who defines the potential and the preference function wishes to satisfy the LTL formula in long term and therefore her interest in making a progress towards the satisfaction of the formula at some point naturally prevails her interest in collecting the rewards. Several examples of $\pot$ and $\pref$ functions that respect this assumption will be shown in Section~\ref{sec:case}.


\section{Solution}
\label{sec:solution}

In this section, we give the details of our solution to Problem~\ref{problem} and prove its correctness and completeness. Discussions on the optimality of the solution are included, too.

\subsection{Offline Indicator Asssignment}

Let $\B_\phi=(S,S_{0},\Sigma,\delta,F)$ be a B\"{u}chi automaton corresponding to the LTL formula $\phi = \varphi \wedge \G \, \F \, \sur$ (Eq.~\ref{eq:formula}) and $\Prod=\T \times \B_\phi=(S_\Prod,S_{\Prod0},\delta_{\Prod},F_{\Prod},W_\Prod)$ the product automaton constructed according to Def.~\ref{def:product}.

Let $S_{\Prod \pi} = \{(q,s)\in S_{\Prod} \mid \sur \in L(q)\}$ denote the subset of states of $\Prod$ that project onto the surveyed states in~$\T$. Furthermore, let $\finf \subseteq F_\Prod$ and $\sinf \subseteq S_{\Prod\pi}$ be the sets of states from which $S_{\Prod \pi}$ and $F_\Prod$ can be visited infinitely many times, respectively. Sets $\finf$ and $\sinf$ can be computed iteratively as the maximal sets of states from which a state in $\sinf$ and $\finf$ is reachable via a finite run of nonzero length, respectively (see Alg.~\ref{alg:value}, lines~\ref{algl:fixpoint1}-\ref{algl:fixpoint2}).

\begin{lemma}\label{th:inftysets}
A run $\varrho_{\Prod}$ of $\Prod$ is accepting iff a state from $\finf$ and a state from $\sinf$ appear in $\varrho_\Prod$ infinitely many times.
\end{lemma}

\begin{proof}
Let $\varrho_{\Prod}=\varrho_{\Prod}(0)\varrho_{\Prod}(1)\ldots $ be an accepting run of $\Prod$, \ie a run with infinitely many visits to $F_{\Prod}$. Note that there is a state in $S_{\Prod \pi}$ that appears in $\varrho_\Prod$ infinitely many times, because $\varrho_\Prod$ satisfies $\phi$ and hence also $\G \, \F \, \sur$. Then there exist infinite index sets $I,J\subseteq \Nat$, where $\varrho_{\Prod}(i)\in F_{\Prod}$, $\varrho_{\Prod}(j)\in S_{\Prod \pi}$, for all $i\in I$ and $j\in J$ . For each state $\varrho_{\Prod}(i)\in F_{\Prod}, i\in I$ there exist infinitely many states $\varrho_{\Prod}(j)\in S_{\Prod \pi}$ where $i<j\in J$, and analogous holds for each state $\varrho_{\Prod}(j)\in S_{\Prod \pi}, j\in J$. Hence, all states $\varrho_\Prod(i),\varrho_\Prod(j)$ where $i\in I, j\in J$ belong to $F_{\Prod}^{\infty},S_{\Prod \pi}^{\infty}$, respectively. 
On the other hand, if a state from $\finf$ occurs on $\varrho_\Prod$ infinitely many times, then $\varrho_\Prod$ is clearly accepting.
\end{proof}

For each state $p\in S_\Prod$ we define the minimum weight of a finite run from $p$ to a state from $\sinf$
\begin{equation}
\distpi (p) = \min_{p' \in \sinf} \distp \big(p,p'\big) 
\label{eq:distpi}
\end{equation}
and the minimum weight of a finite run from $p$ to $\sinf$ containing a state~$p' \in \finf$
\begin{equation}
W_{\Prod F \pi}^* (p,p') =  \min\limits_{p'' \in \sinf} \Big(\distp \big(p,p'\big)+\distp \big(p',p'')\Big).
\label{eq:Wpipp}
\end{equation}
Moreover, we define
\begin{equation}
\distvarphi (p) = \Big(\distp \big(p,p'), W_{\Prod F \pi}^* (p,p')\Big)
\label{eq:distvarphi1}
\end{equation}
where $p' \in \finf $ minimizes $\distp (p,p')$ among the set of states that minimize Eq.~\ref{eq:Wpipp}.
Given $\distvarphi (p_1) = (u_1,v_1)$ and $\distvarphi (p_2) = (u_2,v_2)$, $\distvarphi (p_1) < \distvarphi (p_2)$ if and only if $u_1 < u_2$ and $v_1 < v_2$.

Note that each state $p \in S_\Prod$ with $\distpi(p) = \infty$ or $\distvarphi(p) = (\infty,\infty)$ cannot occur on any accepting run of~$\Prod$. Therefore, 
we assume from now on that $\Prod$ contains only states $p\in S_\Prod$ with $\distpi(p) \neq \infty$ and $\distvarphi(p) \neq (\infty,\infty)$.

\begin{lemma}
$\forall p\in S_\Prod \setminus S_{\Prod\pi}^{\infty}, \exists p' \in S_\Prod: (p,p'), \in \delta_\Prod$, $\distpi(p) > \distpi(p')$, and 
$\forall p\in S_\Prod \setminus F_{\Prod}^{\infty}, \exists p' \in S_\Prod: (p,p'), \in \delta_\Prod, \distvarphi(p) > \distvarphi(p')$.
\end{lemma}

\begin{proof}
Follows directly from Eq.~\ref{eq:distpi}, \ref{eq:Wpipp} and \ref{eq:distvarphi1}.
\end{proof}

We are now ready to define the \emph{shortening indicator functions}
$I_{\Prod \pi},I_{\Prod \phi}\colon \delta_{\Prod}\to \{1,0\},$
which indicate whether a transition induces progress towards the set $\sinf$ and towards both the set $\finf$ and the set $\sinf$ via a state in $\finf$, respectively.
\begin{equation}
I_{\Prod x}\big((p,p')\big)= \begin{cases}
1 & \text{if } \dist_{\Prod x}\big(p\big)> \dist_{\Prod x}\big(p'\big),\\
0 & \text{otherwise,}
\end{cases}
\label{eq:ipi}
\end{equation}
where $x \in \{\pi, \phi\}$.

\begin{corollary}
$\forall p \in S_\Prod \setminus S_{\Prod \pi}^\infty, \exists (p,p')\in \delta_\Prod$, such that $I_{\Prod \pi}\big((p,p')\big) = 1$ and $\forall p \in S_\Prod \setminus F_{\Prod}^\infty, \exists (p,p')\in \delta_\Prod$, such that  $I_{\Prod \phi}\big((p,p')\big) =~1$.
\end{corollary}

The outline of the indicator assignment procedure for the product automaton $\Prod$ is summarized in Alg.~\ref{alg:value}.

\begin{algorithm}[!t]
\small
\caption{Indicator assignment algorithm}
\label{alg:value}
\begin{algorithmic}[1]
\INPUT{$\Prod=(S_\Prod,S_{\Prod0},\delta_{\Prod},F_{\Prod},W_\Prod)$}
\OUTPUT{$\Prod=(S_\Prod,S_{\Prod0},\delta_{\Prod},F_{\Prod},W_\Prod)$, $I_{\Prod \pi}, I_{\Prod \phi}$}
\STATE $F_{\Prod}^{\infty}:=F_{\Prod}, S_{\Prod \pi}^{\infty}:=S_{\Prod \pi}$
\WHILE {fix-point of $F_{\Prod}^{\infty}, S_{\Prod \pi}^{\infty}$ not found}
\label{algl:fixpoint1}
\FORALL {$ p \in F_{\Prod}^{\infty}$, s.t. $\min\limits_{(p,p')\in \delta_\Prod, p'' \in \sinf} \distp\big(p',p''\big)=\infty$}
\STATE remove $p$ from $F_{\Prod}^{\infty}$
\ENDFOR
\FORALL {$ p \in S_{\Prod \pi}^{\infty}$, s.t. $\min\limits_{(p,p')\in \delta_\Prod, p'' \in F_{\Prod}^{\infty}} \distp \big(p',p''\big)=\infty$}
\STATE remove $p$ from $S_{\Prod \pi}^{\infty}$
\ENDFOR
\ENDWHILE
\label{algl:fixpoint2}
\FORALL{$p \in S_\Prod$, s.t. $\distpi \big(p\big)=\infty \, \vee \, \distvarphi \big(p\big)=(\infty,\infty)$}
\STATE remove $p$ together with incident transitions
\ENDFOR
\label{algl:trimstates}
\FORALL{$(p,p') \in \delta_\Prod$}
\STATE compute $I_{\Prod \pi}\big((p,p')\big),I_{\Prod \phi}\big((p,p')\big)$ (Eq.~\ref{eq:ipi})
\ENDFOR
\end{algorithmic}
\end{algorithm}

\subsection{Online Planning}

The online planning algorithm is run at each $t_k$, such that $q_0\ldots q_k$ is the executed run prefix so far (\ie till the current time $t_k$) and it determines the next state $C(q_0\ldots q_k)$ of $\T$ to be visited. 
Simply put, we plan in the product automaton $\Prod$ and then we project the planned onto $\T$. Formally, $\T$ starts in its initial state $q_0$ and $\Prod$ in its initial state $(q_0,s_0)$. For each run prefix $(q_0,s_0)\ldots (q_k,s_k)$ of $\Prod$, the algorithm computes the next state of $\Prod$, denoted by $C_\Prod\big((q_0,s_0)\ldots (q_k,s_k)\big) = (q_{k+1},s_{k+1})$. The next state of $\T$ is 
$C(q_0\ldots q_k) = q_{k+1}$.

To guarantee that the control strategy $C$ generates a run of $\T$ satisfying $\phi$, it is sufficient to ensure that the control strategy $C_\Prod$ generates a run of $\Prod$ that visits $F_\Prod$ infinitely many times. In $\T$, the high value of the preference function $\pref$ was used to guide the robot towards $\sur$. Projected into the product automaton, the high value of $\pref$ can "send" the robot towards a state in $\sinf $. We expand this idea and use the preference function to guide the robot not only towards $\sinf$, but also towards $\finf$. This way, we ensure that $\finf$ is indeed visited infinitely many times.

In particular, we introduce two subgoals in $\Prod$. The first one is the \emph{mission subgoal}, when a visit to $\finf$ is targeted. The second one is the \emph{surveillance subgoal}, when we aim to visit $\sinf$. At each time, one of the subgoals is to be achieved and once it is, the subgoals are switched and the other one is to be achieved. Progress towards both subgoals is governed by maximization of the \emph{attraction function} $\attr$ which we define for the product automaton in analogous way as the cost function (Eq.~\ref{eq:opt}) for Problem~\ref{problem}.

Consider the product $\Prod$ 
obtained after the execution of the offline preparation algorithm (Alg.~\ref{alg:value}). Assume, that $\phi$ is satisfiable, \ie that $\finf$ and $\sinf$ are both nonempty and $(q_0,s_0)\in S_{\Prod}$.
The product $\Prod$ naturally inherits the rewards from~$\T$, \ie $R_\Prod\big((q,s),(q_0,s_0)\ldots(q_k,s_k)\big)=R(q,q_0\ldots q_k)$. Thus, the value of $\pot$ function can be computed on the product automaton (or, more precisely, on its underlying TS $\T_\Prod$) using $R_\Prod$. We use $\pot_\Prod(p,\runprodp,h)$ to denote the value of the state potential function for a state $p$ computed on~$\Prod$.

The value of the attraction $\attr : S_{\Prod}\times S_{\Prod}^{+}\times \Real_{>0}\to \Real_{\geq 0}$ is computed differently for both subgoals.
Initially, the subgoal to be achieved is the surveillance one and the attraction is
\begin{align}
\label{eq:attr}
\attr & \big(p,\runprodp,h\big) = \nonumber \\  & \pot_\Prod \big(p,\runprodp,h\big) + I_{\Prod \pi}\big((p_k,p)\big) \cdot \pref(\alpha(\runprodp)),
\end{align}
where $\runprodp = p_0\ldots p_k$, $(p_k,p) \in \delta_\Prod$. For any run prefix $p_0\ldots p_k$, let $C_\Prod(p_0\ldots p_k)$ be the state with the highest attraction (if there are more of them, we choose one randomly).
Hence, if the attraction of a state that is not closer to the subgoal is higher than the attraction of ones that are, the collection of rewards is preferred and vice versa. However, note that repeated choices of the states that maximize $\attr$ together with Assump.~\ref{assump} guarantee, that the surveillance subgoal, \ie a visit to $\sinf$ will be eventually achieved. Once it is, the mission subgoal becomes the one to be reached.

For the mission subgoal, the attraction needs to be defined in a different way. The reason is that with an analogous definition as for the surveillance subgoal, we would not be able to ensure eventual visit to $\finf$. Intuitively, if $\sur$ was repeatedly unintentionally visited, the value of $\pref(\alpha(\runprodp))$ might not overgrow the value of $\pot_\Prod \big(p,\runprodp,h\big)$, the "non-shortening" transitions might be always chosen to follow and a visit to $\finf$ might be infinitely postponed.

Thus, we define a projection function $\bar\alpha$ that projects a run prefix $\runprodp$ of $\Prod$ onto the corresponding run of $\T$ while removing $\sur$ from some of the states. In particular, on $\bar\alpha(\runprodp)$, the proposition $\sur$ appears at most once in between two successive visits to an accepting state in $\finf$.

\begin{definition}[Projection $\bar{\alpha}$]
Let $\bar\T=(\bar{Q},q_0,\bar{T},\Pi,\bar{L},\bar{W})$ be a transition system, where
$\bar Q = Q \cup \{\bar{q} \mid q \in Q\}$;
if $(q,q') \in T$, then $(q,q'), (\bar{q},q'), (q,\bar{q'}), (\bar{q},\bar{q'}) \in \bar T$ and $\bar W(q,q')=\bar W(\bar q,q')=\bar W(q,\bar{q'})=\bar W(\bar q, \bar{q'})= W(q,q')$; and
$\bar{L}(q) = L(q)$, and $\bar L(\bar q) = L(q)\setminus \{\sur\}$, for all $q \in Q$;
Let $\runprodp=(q_0,s_0) \ldots (q_k,s_k)$ be a run prefix of $\Prod$. 
$\barrunp(0) =q_0$;
$\barrunp(i) = q_i$ if $\sur \not \in L(q_i)$ or $\sur \in L(q_i)$ and $\exists j<i$, such that $(q_j,s_j) \in \finf$ and $\sur \not \in L(q_l)$, for all $j\leq l < i$; and
$\barrunp(i) = \bar q_i$ otherwise.
\end{definition}

The definition of the attraction for the mission mode is
\begin{align}
\label{eq:attr2}
\attr & \big(p,\runprodp,h\big) = \nonumber \\ & \pot_\Prod \big(p,\runprodp,h\big) + I_{\Prod \phi}\big((p_k,p)\big) \cdot \pref(\bar\alpha(\runprodp)),
\end{align}
where $\runprodp = p_0\ldots p_k$, $(p_k,p) \in \delta_\Prod$. 
Similarly as for the surveillance subgoal, the state $C_\Prod(p_0\ldots p_k)$ is the state maximizing attraction (if there are more of them, we choose one randomly). The construction of the attraction together with Assump.~\ref{assump} ensure that the mission subgoal is always eventually reached. Once it is, we aim for the surveillance subgoal again. If both of the subgoals are reached simultaneously, the surveillance subgoal is set to be reached.

The outline of the solution to Problem~\ref{problem} is given in Alg.~\ref{alg:solution}.

\begin{algorithm}[!h]
\small
\caption{Solution to Problem~\ref{problem}}
\label{alg:solution}
\begin{algorithmic}[1]
\INPUT{$\T, \sur, v, R, h, \pot, \phi, \pref$}
\OUTPUT{Control strategy $C$}
\STATE compute $\B_\phi$, $\Prod=\T \times \B_\phi$ and run Alg.~\ref{alg:value}
\IF{$F_{\Prod}^{\infty}=\emptyset$ or $(q_0,s_0)\not \in S_\Prod$}
\RETURN "Mission cannot be accomplished".
\ENDIF
\STATE $\runprodp := (q_0,s_0)$, subgoal $ := \sur$, $k :=0$
\label{alg:online1}
\WHILE {true}
\FORALL{$p$, s.t. $\big(p_k,p\big)\in \delta_{\Prod}$}
\STATE compute $\attr\big(p,\runprodp,h \big)(\text{Eq.}~\ref{eq:attr} \text{ if subgoal } = \sur$ \\ $\text{ and Eq.}~\ref{eq:attr2} \text{ if subgoal } = \phi)$
\ENDFOR
\STATE $C_\Prod(\runprodp):=p$ maximizing $\attr\big(p,\runprodp,h)$
\STATE $C(\alpha(\runprodp)):= \alpha\big(C_\Prod(\runprodp)\big)$
\IF {subgoal $ = \sur$ and $C_\Prod(\runprodp) \in \sinf$}
\STATE { subgoal $:= \phi$}
\ENDIF
\IF {subgoal $ = \phi$ and $C_\Prod(\runprodp) \in \finf$}
\STATE { subgoal $:= \sur$}
\ENDIF
\STATE concatenate $C_\Prod(\runprodp)$ to $\runprodp$; $k := k+1$
\ENDWHILE
\label{alg:online2}
\end{algorithmic}
\end{algorithm}

\subsection{Discussion}
\label{sec:solution_cor}
In this section, we prove that under Assump.~\ref{assump}, our algorithm is correct and complete with respect to the satisfaction of the LTL formula (condition (i) of Problem~\ref{problem}). We discuss  the sub-optimality of the solution and we introduce an assumption, under which the local plan is optimal with respect to condition (ii) of Problem~\ref{problem} among the solutions that do not cause an immediate, unrepairable violation of $\phi$.

\begin{theorem}[Correctness and Completness]
Alg.~\ref{alg:solution} returns a strategy $C$ that generates a run of $\T$ satisfying $\phi$ if and only if such a strategy exists.
\end{theorem}

 \renewcommand{\arraystretch}{1.4}
\begin{table*}
\begin{tabular}{|l|}
\hline
$\textcolor{blue}{\pot_1}(q,q_{0}\ldots q_{k},h) =  \max \limits_{\run(0)\ldots \run(n) \in P_\mathrm{fin}(q,q_{k},h)} \sum_{i=0}^{n} f_1\big(\run(i),q_{0}\ldots q_{k},\rho(0)\ldots\rho(n)\big)$, and\\
$\textcolor{blue}{\pot_{2}}(q,q_{0}\ldots q_{k},h) = $ $\max \limits_{\run(0)\ldots \run(n) \in P_\mathrm{fin}(q,q_{k},h)} \Big(\max\limits_{i=0,\ldots ,n} f_{2}(\run(i),q_{0}\ldots q_{k},\rho(0)\ldots\rho(n))\Big),$\\ where
$f_{1,2}(\run(i),q_{0}\ldots q_{k},\rho(0)\ldots\rho(n)) = R(\run(i),q_{0}\ldots q_{k})-W(\run(0)\ldots \run(i))$ if this value $>0$,  $\run(i)\neq q_k$ and $\run(j)\neq \run(i)$ for all $j<i$, \\ and $f_1=15, f_2 =0$ otherwise. \\ \hline
$\textcolor{blue}{\pref_{1}}(q_{0}\ldots q_{k}) = $ 0 if  $W(q_{i_{\pi}}\ldots q_{k})\leq 50$, and $\maxpot(q_{0}\ldots q_{k},h) +1$ otherwise,\\
$\textcolor{blue}{\pref_{2}}(q_{0}\ldots q_{k}) =$ $\frac{1}{50^{3}}\cdot W(q_{i_{\pi}}\ldots q_{k})^{3}\cdot \maxpot (q_{0}\ldots q_{k},h)$, and
$\textcolor{blue}{\pref_{3}}(q_{0}\ldots q_{k}) = \frac{1}{\sqrt[3]{50}}\cdot \sqrt[3]{W(q_{i_{\pi}}\ldots q_{k})}\cdot  \maxpot (q_{0}\ldots q_{k},h)$, \\ where
$i_{\pi}$ is maximal $0\leq i\leq k$, such that $q_{i_{\pi}}\in \sur$ and $\maxpot(q_{0}\ldots q_{k},h)$ is the maximal $\pot(q,q_{0}\ldots q_{k},h)$ among all $q$, where $(q_{k},q)\in T$\\ \hline
\end{tabular}
\caption{Definitions of the state potential and the preference functions used in the case study.}
\label{tab:pot}
\end{table*}

\begin{proof}
\emph{(Sketch.)} Assume that Alg.~\ref{alg:solution} returns "Mission cannot be accomplished.". Then $\finf$ is empty and according to Lemma~\ref{th:inftysets}, $\phi$ cannot be satisfied in~$\T$.
Assume that Alg.~\ref{alg:solution} computes a strategy $C_\Prod$ for the product $\Prod$. We will show by contradiction that $C_\Prod$ generates a run $\runprod$ of $\Prod$ visiting $\finf$ infinitely many times. 
 Assume that there is a finite prefix $\runprodp = p_0\ldots p_k$ of $\runprod$, such that $p_n\not \in \finf$, for all $n\geq k$ and first, assume that the current subgoal is the surveillance one. Then, according to Assump.~\ref{assump} and the definition of the attraction function, the value of $\pref(\alpha(\runprodp')) > \pot_\Prod(p,\runprodp',h)$ for all prefixes $\runprodp' = p_0\ldots p_k \ldots p_l$ of the run $\runprodp$, such that $l \geq m$, for some $m\geq k$. This means that the "shortening" transitions will be preferred over the "non-shortening" ones since $t_m$ and thus, $p_j \in \sinf$ will be reached eventually. Second, assume that the mission subgoal is the current one. 
Then, according to Assump.~\ref{assump} and the definition of the attraction function, the value of $\pref(\bar\alpha(\runprodp')) > \pot_\Prod(p,\runprodp',h)$ for all prefixes $\runprodp' = p_0\ldots p_k \ldots p_l$ of the run $\runprodp$, such that $l \geq m$, for some $m\geq k$. Similarly as in the previous case, $p_j \in \finf$ will be reached eventually. Thus the proof is complete.
\end{proof}

In general, the satisfaction of condition (ii) of Problem~\ref{problem} cannot be guaranteed as repeated visits to the state maximizing Eq.~\ref{eq:opt} might prevent the mission to be satisfied. However, we reach some level of optimality as disscussed bellow.

In the attraction definition (Eq.~\ref{eq:attr}), the value of the state potential function $\pot_\Prod(p,p_0\ldots p_k,h)$ is computed in the product automaton instead of the transition system. As a result, it is computed assuming that only sequences of transitions that do not cause an immediate, unrepairable violation of the formula can be followed from $q$. 
If the current subgoal of the online planner is the surveillance subgoal, the following optimality statement can be made: A state of $\Prod$ maximizing the attraction (Eq.~\ref{eq:attr}) projects onto the state of $\T$ maximizing the cost function (Eq.~\ref{eq:opt}) taking into consideration only finite runs that do not cause an immediate violation of the formula. In contrast, if the current subgoal of the online planner is the mission one, we cannot claim the similar. First, the indicator function in the attraction (Eq.~\ref{eq:attr}) does not indicate whether a transition of the product automaton leads closer to $\sur$, it rather indicates whether it leads closer to both an accepting state and $\sur$. Second, the preference function in the attraction function (Eq.~\ref{eq:attr}) is computed for $\bar{\alpha}(p_0\ldots p_k)$ instead for $\alpha(p_0\ldots p_k)$. This is necessary for correctness of the algorithm, however, as a result, the value of $\pref(\bar\alpha(p_0\ldots p_k))$ in the attraction (Eq.~\ref{eq:attr}) might be different than the corresponding value of $\pref(q_0\ldots q_k)$ in the cost function (Eq.~\ref{eq:opt}).

In case $\finf = \{ q' \in S_\Prod \mid  q \in \sinf  \text{ and } (q,q') \in \delta_\Prod\}$, the mission subgoal is reached always exactly one planning step after the surveillance subgoal is reached. Therefore, we can reach the optimality that was stated in the previous paragraph for the surveillance subgoal also for the mission subgoal, since all the transitions from $\sinf$ are always "shortening" with respect to $\finf$. In particular, this is the case if a B\"uchi automaton with the property that all the transitions leading to an accepting states are labeled with a set containing $\sur$, is used in the product automaton construction. 
For instance, a surveillance fragment of LTL defined in~\cite{yushan-icra2012} guarantees existence of such a BA. The fragment includes LTL formulas that require to repeatedly visit a surveillance proposition $\sur$ (called an optimizing proposition in~\cite{yushan-icra2012}) and to visit a given set of regions in between any two successive visits to states satisfying $\sur$. In addition, ordering constraints, request-response properties, and safety properties are allowed.

\subsubsection*{Complexity}
The size of a BA for an LTL formula $\phi$ is $2^{\mathcal O(|\phi|)}$ in the worst case, where $|\phi|$ denotes the length of the formula $\phi$~\cite{FastLtlToBa}. However, note that the actual size of the BA is in practice often quite small. The size of the product automaton $\Prod$ is $\mathcal O (|Q|\cdot 2^{\mathcal O (|\phi|)})$. A simple modification of the Floyd-Warshall algorithm is employed to find the minimum weights between each pair of states 
in $\mathcal O (|\Prod|^3)$. The same complexity is reached for the computation of $\finf,\sinf$, $\distpi$ and $\distvarphi$. The shortening indicators $I_{\Prod \pi},I_{\Prod \phi}$ can be computed in linear time and space with respect to the size of $\Prod$. The overall complexity of Alg.~\ref{alg:value} is $\mathcal O \big((|Q|\cdot 2^{\mathcal O (|\phi|)})^3\big)$. The complexity of the online planning algorithm highly depends on the complexity of the state potential and the preference functions. The set $P_\mathrm{fin}(q,q_{k},h)$ can be computed in $\mathcal O(d^h)$, where $d$ denotes the maximal out-degree of states of $\Prod$. If $\pot$ and $\pref$ functions took constant time to compute, the online planning algorithm would be in $\mathcal{O} \big( d \cdot d^h\big)$ per iteration.


\section{Example}
\label{sec:case}

We implemented the framework with several concrete choices of the state potential and the preference function in a Java applet~\cite{tool}. In this section, we report on simulation results to illustrate employment of our approach.

We consider a data gathering robot in a grid-like partitioned environment modeled as a TS depicted in Fig.~\ref{fig:grid}. 
The robot collects data packages of various, changing sizes (rewards) in the visited regions. The following is known about the reward dynamics: A non-negative natural reward can appear in a state with the current reward equal to 0. The probability of the fresh reward being from $\{0,\ldots,15\}$ is 50\% as well as from $\{16,\ldots,60\}$ (\ie the smaller-sized data packages are more likely to occur). The reward drops by 1 every time unit as the data outdate. The visibility range $\vis$ is 6. For example, in Fig.~\ref{fig:grid} the visibility region $\Vis(q_0)$ for the current state $q_0$ is depicted as the blue-shaded area.

\begin{figure}[!t]
\begin{center}
\scalebox{0.23}{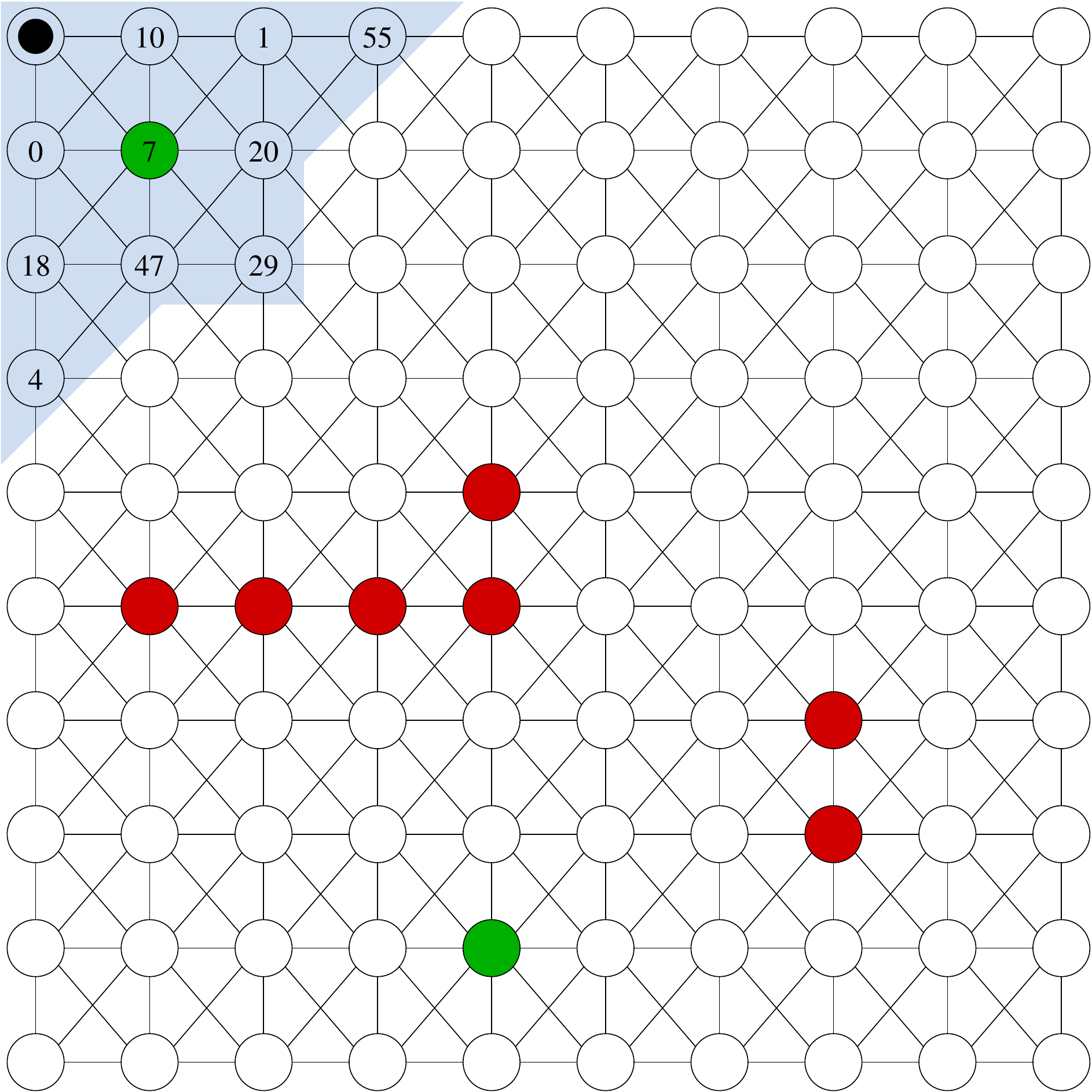}
\caption{A transition system representing the robot (illustrated as the black dot) motion model in a partitioned environment. Individual regions are depicted as nodes (states). Transmitters are in green (labeled with propositions $a$ and $b$, respectively), unsafe locations (labeled with $u$) are in red. The set of transitions contains every pair $(q,q')$ of vertically, horizontally or diagonally neighboring states. 
Weights of a horizontal and a vertical transition are 2, weight of a diagonal transition is 3.
}
\label{fig:grid}
\end{center}
\end{figure}

The mission assigned to the robot is to alternately visit the two transmitters (in green, labeled with propositions $a$, and $b$, respectively), while avoiding unsafe locations (in red, labeled with $u$). The surveillance proposition $\sur$ is true in both transmitter regions.
The LTL formula for the mission is
\begin{align*}
\phi \:\equiv \: &  \G\,\big(a \Rightarrow\X\, (\neg a\,\U\,b)\big) \ \wedge \
\G \, \big(b\Rightarrow \X \, (\neg b\,\U\,a )\big)\ \wedge \\ & \G(\neg u) \ \wedge \ \G\,\F\,\sur.
\end{align*}

In our simulations, we consider the planning horizon $h=9$ and several variants of the state potential function and the preference function that are summarized in Table~\ref{tab:pot}. The first state potential function $\pot_1$ is the maximal sum of rewards that can be collected on a finite run while taking into account the reward behavior assumptions described above. If the run visits a state more than once or a reward of a state drops below 0, we assume the reward there is 15. The second state potential function $\pot_{2}$ is defined as the maximal size of a single data package that can be collected on a finite run.

The respective ratio of the value of $\pref$ and the maximum value of $\pot$ is always non-decreasing with the time elapsed since last transmission and the value of $\pref$ overgrows the maximum value of $\pot$ when the elapsed time is 50.
Intuitively, $\pref_1$ sets zero importance on going towards a transmitter if the last transmission occurred not more then 50 time units ago. On the other hand, $\pref_2$ rises quite slowly at the beginning and very quickly later. In contrast, the function $\pref_3$ grows very fast in the beginning and its growth slows down. 

For each of 6 instances we executed 5 runs of 100 iterations of the online planner. The sizes of the data collected in time are depicted in Fig.~\ref{fig:plots}. 
Table~\ref{tab:reward} shows the mean of average reward per transition and the time between consecutive surveys, respectively. 
As expected, the faster the preference function grows with time since the last survey, the smaller the reward per transition and the shorter the time between consecutive transmissions are. For $\pref_{1}$ and $\pref_{2}$, the difference in the reward per transition is not high, since in both cases the collection of rewards is preferred in the beginning, whereas $\pref_{3}$ is very steep and therefore drives the robot towards transmitter quickly. Function $\pot_{1}$ computing the maximal sum of rewards that can be collected gives, as expected, higher average and lower variance for both objectives comparing to $\pot_{2}$ that aims to collect big packages.

\begin{figure}[!t]
\centering
\subfloat[$\pot_{1}$ and $\pref_1$]{\label{fig:pot1pref1}\includegraphics[width=0.22\textwidth]{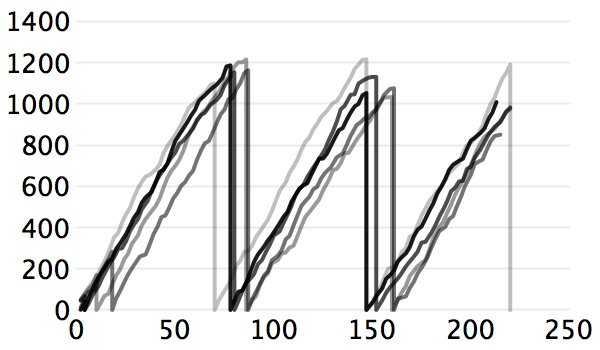}}~
\subfloat[$\pot_{2}$ and $\pref_1$]{\label{fig:pot2pref1}\includegraphics[width=0.22\textwidth]{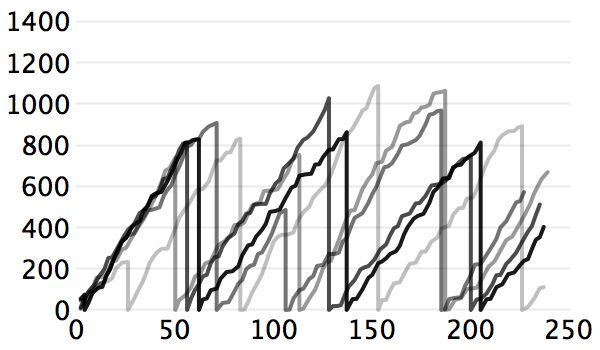}}\\
\subfloat[$\pot_{1}$ and $\pref_2$]{\label{fig:pot1pref2}\includegraphics[width=0.22\textwidth]{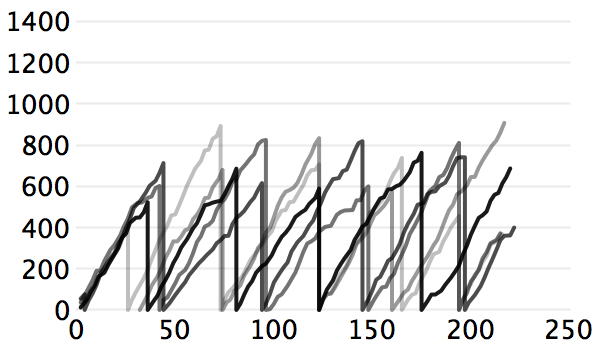}}~
\subfloat[$\pot_{2}$ and $\pref_2$]{\label{fig:pot2pref2}\includegraphics[width=0.22\textwidth]{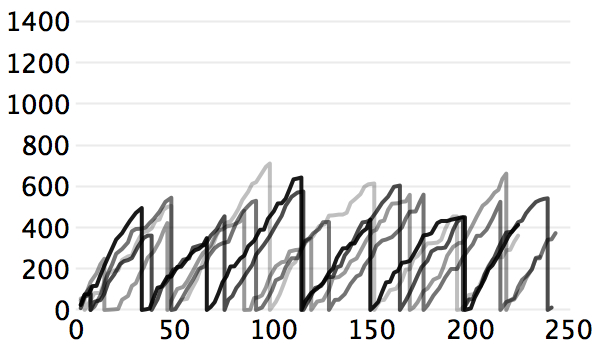}}\\
\subfloat[$\pot_{1}$ and $\pref_3$]{\label{fig:pot1pref3}\includegraphics[width=0.22\textwidth]{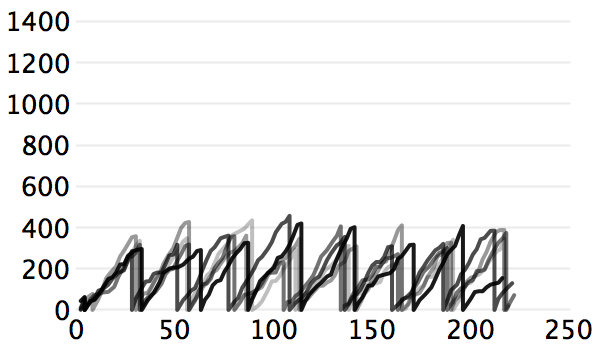}}~
\subfloat[$\pot_{2}$ and $\pref_3$]{\label{fig:pot2pref3}\includegraphics[width=0.22\textwidth]{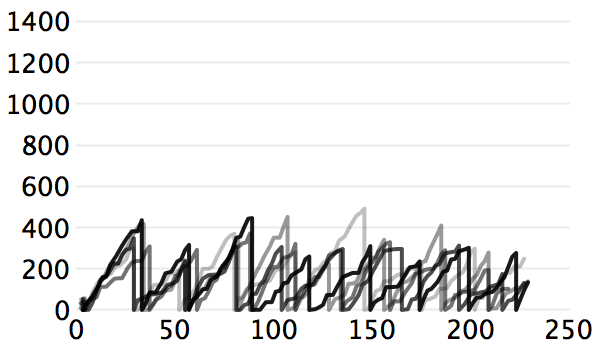}}
\caption{Total size of data collected since the last transmission with respect to time depicted for each executed run.}
\label{fig:plots}
\end{figure}

\renewcommand{\arraystretch}{1.3}
{
\begin{table}[!t]
\centering
\begin{tabular}{|c| c || c | c | c | c | c | c |}
\hline
& & 1/1 & 1/2 & 1/3 & 2/1 & 2/2 & 2/3\\
\hline
\hline
\multirow{3}{*}{r/T} & AVG & 33.8 & 33.7 & 28.2 & 30.8 & 29.2 & 25.9\\
& $\nu$ & 2.4\% & 4.4\% & 6.0\% & 3.9\% & 5.8\% & 8.3\%\\
\cline{2-8}
 & VAR & 13.8 & 14.6 & 15.9 & 19.2 & 18.3 & 18.7\\
\hline
\multirow{3}{*}{t} & AVG  & 73.4 & 46.0 & 26.7 & 66.2 & 41.9 & 26.4\\
& $\nu$ & 2.4\% & 8.0\% & 2.1\% & 8.4\% & 7.7\% & 3.0\%\\
\cline{2-8}
 & VAR  & 2.8 & 6.0 & 2.7 & 11.2 & 6.5 & 3.3\\
\hline
\end{tabular}
\caption{Statistical results for the reward per~transition (r/T) and the time between consecutive surveys (t) for different choices of $\pot/\pref$ functions (in the header). AVG is the mean of average computed on each run and VAR the mean of variance computed on each run. $\nu$ shows the percentage variance of the average among the runs.
}
\label{tab:reward}
\end{table}

The experiments were run on Mac OS X 10.7.3 with 2.7 GHz Intel Core i5 and 4 GB DDR3 memory. The BA had 8 states (3  accepting) and it satisfied the condition for optimality from Sec.~\ref{sec:solution_cor}. The product automaton had 800 states. The offline part of Alg.~\ref{alg:solution} took 6 seconds and one iteration of the online planning algorithm 1-2 milliseconds.


\section{Conclusions and Future Work}

We proposed a general framework for robot motion planning in environment with dynamically changing rewards. While a high-level surveillance mission is guaranteed to be accomplished, the user-defined priorities on trade-off between the surveillance frequency and the reward collection are taken into account. The motion of the robot is modeled as a weighted transition system. Although the weights are in this paper interpreted as time durations of the transitions, they can be, in general interpreted, as any quantitative aspect, such as length or cost. In future work, we would like to extend the suggested framework for systems that are modeled as Markov decision processes and to reaching solution optimality for special subclasses of the reward dynamics. Our plan is also to extend the implementation of the framework.

\label{sec:conclusion}


\section{Acknowledgement}
We thank Calin Belta and Dennis Ding from Boston University for many useful discussions.


\bibliographystyle{IEEEtran}
\bibliography{rh_attraction_CDC2012}

\begin{thebibliography}{10}
\providecommand{\url}[1]{#1}
\csname url@samestyle\endcsname
\providecommand{\newblock}{\relax}
\providecommand{\bibinfo}[2]{#2}
\providecommand{\BIBentrySTDinterwordspacing}{\spaceskip=0pt\relax}
\providecommand{\BIBentryALTinterwordstretchfactor}{4}
\providecommand{\BIBentryALTinterwordspacing}{\spaceskip=\fontdimen2\font plus
\BIBentryALTinterwordstretchfactor\fontdimen3\font minus
  \fontdimen4\font\relax}
\providecommand{\BIBforeignlanguage}[2]{{%
\expandafter\ifx\csname l@#1\endcsname\relax
\typeout{** WARNING: IEEEtran.bst: No hyphenation pattern has been}%
\typeout{** loaded for the language `#1'. Using the pattern for}%
\typeout{** the default language instead.}%
\else
\language=\csname l@#1\endcsname
\fi
#2}}
\providecommand{\BIBdecl}{\relax}
\BIBdecl

\bibitem{lavalle2006planning}
S.~M. LaValle, \emph{Planning Algorithms}.\hskip 1em plus 0.5em minus
  0.4em\relax Cambridge Univ. Press, 2006.

\bibitem{Antoniotti95}
M.~Antoniotti and B.~Mishra, ``Discrete {E}vent {M}odels + {T}emporal {L}ogic =
  {S}upervisory {C}ontroller: {A}utomatic {S}ynthesis of {L}ocomotion
  {C}ontrollers,'' in \emph{Proceedings of {IEEE} {ICRA}}, 1995, pp.
  1441--1446.

\bibitem{Loizou04}
S.~G. Loizou and K.~J. Kyriakopoulos, ``Automatic {S}ynthesis of {M}ultiagent
  {M}otion {T}asks {B}ased on {LTL} {S}pecifications,'' in \emph{Proceedings of
  {IEEE} {CDC}}, 2004, pp. 153--158.

\bibitem{marius-tac2008}
M.~Kloetzer and C.~Belta, ``{A} {F}ully {A}utomated {F}ramework for {C}ontrol
  of {L}inear {S}ystems from {T}emporal {L}ogic {S}pecifications,''
  \emph{{IEEE} {T}ransactions on {A}utomatic {C}ontrol}, vol.~53, no.~1, pp.
  287--297, 2008.

\bibitem{sertac-cdc2009}
S.~Karaman and E.~Frazzoli, ``Sampling-based {M}otion {P}lanning with
  {D}eterministic $\mu$-{C}alculus {S}pecifications,'' in \emph{Proceedings of
  {IEEE} {CDC}}, 2009, pp. 2222--2229.

\bibitem{hadas09TL}
H.~Kress-Gazit, G.~E. Fainekos, and G.~J. Pappas, ``Temporal {L}ogic-based
  {R}eactive {M}ission and {M}otion {P}lanning,'' \emph{IEEE Transactions on
  Robotics}, vol.~25, no.~6, pp. 1370--1381, 2009.

\bibitem{fainekos09TL}
G.~E. Fainekos, A.~Girard, H.~Kress-Gazit, and G.~J. Pappas, ``Temporal {L}ogic
  {M}otion {P}lanning for {D}ynamic {R}obots,'' \emph{Automatica}, vol.~45,
  no.~2, pp. 343--352, 2009.

\bibitem{LaAnBe-ACC10}
M.~Lahijanian, S.~B. Andersson, and C.~Belta, ``{C}ontrol of {M}arkov
  {D}ecision {P}rocesses from {PCTL} {S}pecifications,'' in \emph{{Proceedings
  of {ACC}}}, 2011, pp. 311 --316.

\bibitem{FirstLtlDef}
A.~Pnueli, ``The {T}emporal {L}ogic of {P}rograms,'' in \emph{Proceedings of
  {IEEE} {FOCS}}, 1977, pp. 46--57.

\bibitem{principles}
C.~Baier, J.-P. Katoen, and K.~G. Larsen, \emph{Principles of Model
  Checking}.\hskip 1em plus 0.5em minus 0.4em\relax MIT Press, 2008.

\bibitem{mpc}
J.~B. Rawlings and D.~D.~Q.~Mayne, \emph{Model Predictive Control Theory and
  Design}.\hskip 1em plus 0.5em minus 0.4em\relax Nob Hill Pub., 2009.

\bibitem{Nok2009}
T.~Wongpiromsarn, U.~Topcu, and R.~M. Murray, ``{R}eceding {H}orizon {T}emporal
  {L}ogic {P}lanning for {D}ynamical {S}ystems,'' in \emph{Proceedings of
  {IEEE} {CDC}/{CCC}}, 2009, pp. 5997--6004.

\bibitem{nok-hscc2010}
------, ``{R}eceding {H}orizon {C}ontrol for {T}emporal {L}ogic
  {S}pecifications,'' in \emph{Proceedings of {HSCC}}, 2010, pp. 101--110.

\bibitem{Dennis2010}
X.~C. Ding, C.~Belta, and C.~G. Cassandras, ``Receding {H}orizon {S}urveillance
  with {T}emporal {L}ogic {S}pecifications,'' in \emph{Proceedings of {IEEE}
  {CDC}}, 2010, pp. 256--261.

\bibitem{FastLtlToBa}
P.~Gastin and D.~Oddoux, ``Fast {LTL} to {B}\"{u}chi {A}utomata
  {T}ranslation,'' in \emph{Proceedings of {CAV}}, 2001, pp. 53--65.

\bibitem{ltl2dstar}
\BIBentryALTinterwordspacing
J.~Klein. (2007) ltl2dstar -- {LTL} to {D}eterministic {S}treett and {R}abin
  {A}utomata. [Online]. Available: \url{http://www.ltl2dstar.de/}
\BIBentrySTDinterwordspacing

\bibitem{yushan-icra2012}
Y.~Chen, J.~Tumova, and C.~Belta, ``{LTL} {R}obot {M}otion {C}ontrol based on
  {A}utomata {L}earning of {E}nvironmental dynamics.'' in \emph{Proceedings of
  {IEEE} {ICRA}}, 2012, pp. 5177--5182.

\bibitem{tool}
\BIBentryALTinterwordspacing
M.~Svorenova, J.~Tumova, J.~Barnat, and I.~Cerna. (2012) Simulation of
  attraction-based approach to receding horizon control with {LTL}
  specification. [Online]. Available:
  \url{http://www.fi.muni.cz/$\sim$x175388/simulationLTLrhc}
\BIBentrySTDinterwordspacing

\end{thebibliography}

\end{document}